\documentclass[11pt]{article}

\usepackage[preprint]{acl}

\usepackage{times}
\usepackage{latexsym}

\usepackage[T1]{fontenc}

\usepackage[utf8]{inputenc}

\usepackage{microtype}

\usepackage{inconsolata}

\usepackage{graphicx}

\usepackage{booktabs}       
\usepackage{amsfonts}       
\usepackage{nicefrac}       
\usepackage{microtype}      
\usepackage{xcolor}         

\usepackage{algorithm}
\usepackage{algpseudocode}
\usepackage{booktabs}
\usepackage{multirow}

\usepackage{amsmath, amssymb}
\usepackage{listings}

\newcommand{\mat}[1]{\mathbf{#1}}

\usepackage{amsthm}

\newtheorem{theorem}{Theorem}

\title{ZO-Act: Efficient Zeroth-Order Fine-Tuning via One-Shot Activation-Informed Low-Rank Subspaces}

\author{
  \begin{tabular}{@{}l@{}}
    \textbf{Xun Dong}$^1$ \quad \textbf{Yibo Xu}$^1$ \quad 
    \textbf{Naigang Wang}$^2$ \quad
    \textbf{Xin Li}$^1$ 
    \quad
     \textbf{Penghang Yin}$^1$ \quad \textbf{Zi Yang}$^{1,\dagger}$
  \end{tabular}
  \\
  $^1$University at Albany, SUNY \quad $^2$IBM T. J. Watson Research Center \\
  \texttt{\{xdong5, yxu25, xli48, pyin, zyang8\}@albany.edu} \\
  \texttt{nwang@us.ibm.com}
}

\begin{document}
\maketitle

\begingroup
  \renewcommand\thefootnote{}
  \footnotetext{$^\dagger$ Corresponding author}
\endgroup

\begin{abstract}

Zeroth-order (ZO) optimization enables fine-tuning large language models when backpropagation is unavailable or memory-prohibitive, but existing methods often perturb full model weights or randomly constructed low-dimensional subspaces, yielding high-variance estimates and limited performance. We propose ZO-Act, an activation-informed ZO fine-tuning method that restricts perturbations to a fixed low-rank subspace derived from input activations. For each linear layer, ZO-Act computes a small activation basis once at initialization and optimizes only lightweight coefficient matrices using forward-only loss evaluations. This reduces the effective perturbation dimension, exposes explicit trainable variables compatible with momentum-based optimizers such as Adam, and naturally supports quantized LLM fine-tuning by keeping low-bit weights frozen. We analyze ZO-Act as zeroth-order optimization over a restricted coefficient space and show that perturbing the low-dimensional coefficients reduces both the variance-dependent convergence term and the finite-difference error of the ZO estimator, at the cost of a controlled subspace approximation bias that is mitigated by the low-rank structure of LLM activations and gradients. Experiments on Llama-3-8B, OPT-13B, and INT4 Llama-3-8B show consistent gains over strong ZO fine-tuning baselines across language understanding, question answering, and commonsense reasoning.
\end{abstract}

\section{Introduction}

Fine-tuning large language models (LLMs) \citep{houlsby2019parameter,Hu2021LoRALA,gurses2025diablo,dettmers2023qlora} has become a standard technique for adapting pretrained models to downstream tasks. 
However, as LLMs scale to billions of parameters, conventional first-order fine-tuning becomes increasingly memory-intensive due to backpropagation, activation storage, and optimizer states. 
Zeroth-order (ZO) optimization provides a promising forward-only alternative, estimating update directions using only loss evaluations and avoiding backpropagation \citep{malladi2023fine}.
Despite its memory advantage, ZO fine-tuning suffers from high gradient-estimation variance when perturbations are sampled in the full parameter space. 
MeZO \citep{malladi2023fine} showed that in-place ZO-SGD can fine-tune very large models with inference-level memory, but later studies identified slow convergence and instability as major limitations. 

While ZO optimization offers an appealing alternative, matching FO optimization in terms of convergence and accuracy is still challenging. The core issue is that gradients estimated from function queries often have high variance, and this variance worsens in higher-dimensional problems \citep{duchi2015optimal,nesterov2017random,liu2018zeroth}.
Subsequent work has improved ZO fine-tuning through better optimizers and more structured perturbations, including ZO-Adam and momentum variants \citep{zhang2024revisiting}, ZO-Muon \citep{lang2026powering}, curvature-aware preconditioning \citep{zhao2025second}, low-rank perturbations \citep{yu2024subzero,chen2025enhancing,lin2026agzo}, and sparse perturbations \citep{liu2026sparse}. A particularly effective direction is low-rank or subspace-based ZO perturbation, which reduces the perturbation dimension to improve estimator stability. 
LOZO \citep{chen2025enhancing} and SubZero \citep{yu2024subzero} restrict perturbations to low-rank matrices to reduce variance. 
AGZO \citep{lin2026agzo} further incorporates activation information into low-rank ZO perturbations, showing that activation-aware directions can provide stronger update signals. 
ZO-Act uses activation information in a different way. Instead of only guiding the construction of perturbation directions, it uses activations to define a fixed low-dimensional parameterization. 
Specifically, each adapted weight update is represented by a frozen activation-informed basis and a small trainable coefficient matrix. 
This turns ZO fine-tuning into explicit subspace optimization, reducing the perturbation dimension, enabling standard momentum-based optimizers such as Adam, avoiding full-weight perturbation materialization, and naturally supporting quantized backbones by keeping the original weights frozen.

We therefore propose ZO-Act, a one-shot activation-informed ZO fine-tuning method. 
For each adapted linear layer, ZO-Act computes a low-rank basis from input activations using a calibration batch, freezes this basis throughout fine-tuning, and optimizes only the corresponding coefficient matrix. 
This one-shot fixed-parameterization design provides a simple forward-only alternative to perturbation-based low-rank ZO methods while preserving the memory advantages needed for large and quantized LLM adaptation.

\begin{itemize}
    \item We propose ZO-Act, a one-shot activation-informed ZO fine-tuning method that uses input activations to define a fixed low-dimensional subspace.
ZO-Act freezes the activation basis and optimizes only lightweight coefficient matrices, turning full-weight perturbation into explicit coefficient-space optimization.
This reduces the effective perturbation dimension, exposes explicit trainable variables compatible with standard momentum-based optimizers such as Adam, and naturally supports quantized LLM fine-tuning by keeping the low-bit weights frozen.

    \item We analyze ZO-Act as zeroth-order optimization over a restricted coefficient space and show that, by perturbing the $k$-dimensional coefficients instead of the $d$-dimensional weights, it reduces both the variance-dependent convergence term and the finite-difference error of the ZO estimator.
    The analysis also makes the resulting trade-off explicit: ZO-Act obtains lower-variance estimation at the cost of a controlled subspace approximation bias, which is mitigated by the low-rank structures of LLM activations and gradients.

    \item We demonstrate that ZO-Act consistently improves over strong ZO baselines on Llama-3-8B, OPT-13B, and INT4 Llama-3-8B across language understanding, question answering, and commonsense reasoning tasks, showing its effectiveness for both full-precision and quantized forward-only fine-tuning.
\end{itemize}

\section{Related Work}

\paragraph{Zeroth-order fine-tuning.}
MeZO \citep{malladi2023fine} adapts simultaneous perturbation stochastic approximation to language model fine-tuning by estimating an update direction from two forward losses evaluated under opposite random perturbations. 
Its in-place implementation avoids storing explicit perturbation vectors and therefore substantially reduces training memory. 
Nevertheless, MeZO samples perturbations in the full parameter space, so the resulting estimator remains highly noisy in billion-dimensional models.

A number of recent works improve this basic ZO fine-tuning pipeline through optimizer design, curvature information, and structured parameter selection. 
Zhang et al. \citep{zhang2024revisiting} provide a systematic study of ZO variants for LLM adaptation, including momentum, Adam-style updates, conservative updates, block-wise descent, and hybrid ZO-FO training. 
HiZOO \citep{zhao2025second} estimates diagonal Hessian information and uses it to precondition ZO updates, addressing the heterogeneous curvature of LLM loss landscapes. 
ZO-Muon \citep{lang2026powering} combines matrix-structured ZO updates with Muon-style optimization to improve the stability and effectiveness of LLM fine-tuning. 
These methods demonstrate that reducing estimator noise and improving update geometry are crucial for making ZO fine-tuning competitive.

\paragraph{Structured ZO perturbations.}
Another closely related line of work reduces the perturbation dimension by imposing structured perturbations. 
S-MeZO \citep{liu2026sparse} reduces the update dimension by selecting a sparse set of sensitive parameters. 
LOZO \citep{chen2025enhancing} and SubZero \citep{yu2024subzero} perform ZO optimization in randomly generated low-dimensional subspaces. ZO-Muon \citep{lang2026powering} further exploits low-rank structure by combining subspace gradient orthogonalization with Muon-style matrix updates, improving the stability and effectiveness of ZO fine-tuning. 
AGZO \citep{lin2026agzo} incorporates activation information into the construction of low-rank ZO perturbations, showing that activation-aware directions can provide more informative update signals than purely random perturbations. ZO-Act is most closely related to this subspace-based line of work, but differs in how the subspace is used. 
Existing methods primarily use low-dimensional subspaces to construct more effective perturbation directions, whereas ZO-Act uses input activations to define a fixed update parameterization. 
The activation-informed basis is computed once and kept frozen, and fine-tuning is performed by optimizing only the lightweight coefficient matrices.

\section{Methodology}
\label{sec:method}


\textbf{Motivation.}
Consider a standard linear layer with input dimension $m$ and output dimension $n$,
$
    \mat{Y} = \mat{X}\mat{W},
$
where $\mat{X}\in \mathbb{R}^{b\times m}$, $\mat{W}\in \mathbb{R}^{m\times n}$, and $\mat{Y}\in \mathbb{R}^{b\times n}$. Let $\mat{g}_{\mat{Y}}\in \mathbb{R}^{b\times n}$ denote the gradient of the final loss with respect to the layer output $\mat{Y}$. Then the gradient with respect to the weight matrix $\mat{W}$ is
$
    \mat{g}_{\mat{W}} = \mat{X}^\top \mat{g}_{\mat{Y}}.
$
Let
$
    \mat{X} = \mat{U}\mat{D}\mat{V}^\top
$
be the thin SVD of $\mat{X}$, where the diagonal entries of $\mat{D}$ are the singular values in descending order. Denote the top-$r$ components by
\[
    \mat{U}_r := \mat{U}_{:,:r}, 
    \quad
    \mat{V}_r := \mat{V}_{:,:r}, 
    \quad
    \mat{D}_r := \mat{D}_{:r,:r}.
\]
The weight gradient can be written as
\[
    \mat{g}_{\mat{W}}
    =
    \mat{V}\mat{D}\mat{U}^\top \mat{g}_{\mat{Y}}.
\]
If the singular values of $\mat{X}$ decay rapidly, then $\mat{g}_{\mat{W}}$ can be well approximated by its rank-$r$ truncation:
\[
    \mat{g}_{\mat{W}}
    \approx
    \mat{V}_r\mat{D}_r\mat{U}_r^\top \mat{g}_{\mat{Y}}.
\]
This suggests that the dominant components of the weight gradient lie in the subspace spanned by the top right singular vectors of the input activation matrix, namely $\mat{V}_r$.

The same intuition can also be understood from the perspective of weight perturbation. Suppose we perturb the weight matrix along the activation-informed subspace:
$
    \Delta \mat{W} = \mat{V}_r \mat{R},
    \qquad
    \mat{R}\sim \mathcal{N}(0, I).
$
Then the induced perturbation on the layer output is
\[
    \Delta \mat{Y}
    =
    \mat{X}\Delta \mat{W}
    =
    \mat{U}\mat{D}\mat{V}^\top \mat{V}_r \mat{R}
    =
    \mat{U}_r\mat{D}_r \mat{R}.
\]
Thus, perturbing $\mat{W}$ within the span of $\mat{V}_r$ directly targets the dominant activation directions of the layer. In contrast, perturbations outside this subspace are largely suppressed by the small singular values of $\mat{X}$ and have limited effect on the layer output. Therefore, restricting zeroth-order perturbations to the $\mat{V}_r$ subspace preserves the most effective perturbation directions while substantially reducing the perturbation dimension.

\noindent \textbf{ZO-Act.}
Motivated by the above observation, we restrict zeroth-order perturbations to the dominant right singular subspace of the input activations. For each linear layer, we first compute a fixed activation-informed basis from a calibration batch, and then perform zeroth-order optimization only over a low-dimensional coefficient matrix.

Consider a linear layer with pretrained weight $\mat{W}\in\mathbb{R}^{m\times n}$. Let $\mat{V}_r\in\mathbb{R}^{m\times r}$ be the top-$r$ right singular vectors of the input activation matrix $\mat{X}$. We parameterize the weight update as
\[
    \Delta \mat{W} = \mat{V}_r \mat{B},
\]
where $\mat{B}\in\mathbb{R}^{r\times n}$ is the only trainable parameter for this layer. The effective weight is therefore
\[
    \mat{W}_{\mathrm{eff}}
    =
    \mat{W} + \mat{V}_r \mat{B}.
\]
The pretrained weight $\mat{W}$ and the activation basis $\mat{V}_r$ are both frozen throughout training, while $\mat{B}$ is initialized as zero.

At each training step, we sample a random perturbation in the low-dimensional coefficient space:
\[
    \mat{Z}\sim \mathcal{N}(0,I),
    \qquad
    \mat{Z}\in\mathbb{R}^{r\times n}.
\]
The corresponding weight-space perturbation is
\[
    \Delta \mat{W}_{\mathrm{ZO}}
    =
    \mat{V}_r \mat{Z}.
\]
Given a perturbation magnitude $\mu>0$, the forward-perturbed effective weight is
\[
    \mat{W}_{\mathrm{eff}}^{+}
    =
    \mat{W} + \mat{V}_r(\mat{B}+\mu \mat{Z}).
\]
For an input activation $\mat{X}$, the perturbed forward pass can be written as
\[
    \mat{Y}^{+}
    =
    \mat{X}\mat{W}
    +
    (\mat{X}\mat{V}_r)(\mat{B}+\mu\mat{Z}).
\]

Let $\mathcal{L}^{+}$ denote the loss from the forward-perturbed model and let $\mathcal{L}$ denote the unperturbed loss on the same mini-batch. The forward-difference zeroth-order estimator for $\mat{B}$ is
\[
    \widehat{\mat{g}}_{\mat{B}}
    =
    \frac{\mathcal{L}^{+}-\mathcal{L}}{\mu}\mat{Z}.
\]
With $q$ independent perturbation directions, we average the estimates:
\[
    \widehat{\mat{g}}_{\mat{B}}
    =
    \frac{1}{q}
    \sum_{j=1}^{q}
    \frac{\mathcal{L}^{(j,+)}-\mathcal{L}}{\mu}
    \mat{Z}^{(j)}.
\]
We then update $\mat{B}$ using a first-order optimizer such as Adam with the estimated gradient $\widehat{\mat{g}}_{\mat{B}}$. Since the perturbation is sampled in $\mathbb{R}^{r\times n}$ rather than $\mathbb{R}^{m\times n}$, the effective perturbation dimension is reduced from $mn$ to $rn$, where $r\ll m$. The whole ZO-Act algorithm is summarized in Algorithm \ref{alg:lozo-init-subspace} and Algorithm \ref{alg:lozo-finetune-short}.


\begin{algorithm}[!t]
\small
\caption{Subspace Initialization}
\label{alg:lozo-init-subspace}
\begin{algorithmic}[1]
\Require Pretrained model, calibration batch $\mathcal{B}_{\mathrm{cal}}$, rank $r$
\State Run one forward pass on $\mathcal{B}_{\mathrm{cal}}$ and collect input activations $\mat{X}_\ell$ for each linear layer $\ell$.
\For{each linear layer $\ell$ with weight $\mat{W}_\ell\in\mathbb{R}^{m_\ell\times n_\ell}$}
    \State Compute the top-$r$ right singular vectors of $\mat{X}_\ell$:
    \[
        \mat{X}_\ell \approx \mat{U}_{\ell,r}\mat{D}_{\ell,r}\mat{V}_{\ell,r}^\top .
    \]
    \State Freeze the pretrained weight $\mat{W}_\ell$ and the basis $\mat{V}_{\ell,r}$.
    
\EndFor
\State \Return Frozen bases $\mat{V}_{\ell,r}$.
\end{algorithmic}
\end{algorithm}

\begin{algorithm}[!t]
\footnotesize
\caption{ZO-Act}
\label{alg:lozo-finetune-short}
\begin{algorithmic}[1]
\Require Data $\mathcal{D}$, pretrained weights $\{\mat{W}_\ell\}_\ell$, rank $r$, perturbation scale $\mu$, learning rate $\eta$, queries $q$, steps $T$

\State Sample calibration batch $\mathcal{B}_{\mathrm{cal}} \subset \mathcal{D}$.

\State Obtain frozen subspaces $\{\mat{V}_{\ell,r}\}_\ell$ using Algorithm~\ref{alg:lozo-init-subspace} with $\mathcal{B}_{\mathrm{cal}}$ and rank $r$.

\State Initialize the trainable matrix
    $\mat{B}_\ell \gets \mat{0}\in\mathbb{R}^{r\times n_\ell}$ for all $\ell$.

\For{$t=1,\dots,T$}
    \State Sample mini-batch $\mathcal{B}_t\subset\mathcal{D}$.
    \State Evaluate $\mathcal{L}$ with $\mat{W}_{\ell,\mathrm{eff}}=\mat{W}_\ell+\mat{V}_{\ell,r}\mat{B}_\ell$.
    \State Initialize $\widehat{\mat{g}}_{\mat{B}_\ell}\gets 0$ for all $\ell$.

    \For{$j=1,\dots,q$}
        \State Sample $\mat{Z}_\ell^{(j)}\sim\mathcal{N}(0,I)$ for all $\ell$.
        \State Evaluate $\mathcal{L}^{(j,+)}$ with 
        $\mat{W}_{\ell,\mathrm{eff}}^{(j,+)}=\mat{W}_\ell+\mat{V}_{\ell,r}(\mat{B}_\ell$  
        \Statex \hspace{\algorithmicindent} \hspace{\algorithmicindent} $+\mu\mat{Z}_\ell^{(j)})$.
        \State $a^{(j)}\gets(\mathcal{L}^{(j,+)}-\mathcal{L})/\mu$.
        \State $\widehat{\mat{g}}_{\mat{B}_\ell}\gets
        \widehat{\mat{g}}_{\mat{B}_\ell}
        +a^{(j)}\mat{Z}_\ell^{(j)}$ for all $\ell$.
    \EndFor

    \State $\widehat{\mat{g}}_{\mat{B}_\ell}\gets
    \widehat{\mat{g}}_{\mat{B}_\ell}/q$ for all $\ell$.
    \State Update each $\mat{B}_\ell$ with an optimizer using $\widehat{\mat{g}}_{\mat{B}_\ell}$.
\EndFor

\State \Return Fine-tuned model with 
$\mat{W}_{\ell,\mathrm{eff}}=\mat{W}_\ell+\mat{V}_{\ell,r}\mat{B}_\ell$.
\end{algorithmic}
\end{algorithm}

\vspace{-15pt}
\section{Theoretical Analysis}
\label{sec:theory}


We analyze ZO-Act as zeroth-order optimization over a restricted coefficient space. 
Let \(F(\theta)\) denote the fine-tuning objective, where \(\theta\in\mathbb{R}^{d}\) collects the weight parameters of all adapted linear layers. 
As described in Section~\ref{sec:method}, ZO-Act freezes the pretrained weights and the activation-informed bases, and optimizes only the coefficient matrices \(\{\mat B_\ell\}_{\ell=1}^{L}\). 
Let \(\beta\in\mathbb{R}^{k}\) collect all trainable coefficient matrices, where \(k=\sum_{\ell=1}^{L} r n_\ell\). 
Equivalently, there exists a fixed activation-informed embedding \(U\) such that \(\theta(\beta)=\theta_0+U\beta\), where \(\theta_0\) denotes the frozen pretrained weights. 
Thus, ZO-Act optimizes the restricted objective \(\phi(\beta)=F(\theta_0+U\beta)\), rather than the full objective \(F(\theta)\) over all weight entries.

This restricted formulation creates a variance--bias trade-off. 
Because ZO-Act samples perturbations only in the \(k\)-dimensional coefficient space, the variance-dependent term in the ZO estimator is governed by \(k\), rather than the full weight dimension \(d\). 
At the same time, restricting updates to the range of \(U\) introduces a subspace approximation bias: ZO-Act can only reduce the component of the gradient captured by the activation-informed update subspace.

We first examine the variance of the estimated gradient in the coefficient space. 
Let \(g_t=\nabla\phi(\beta_t)\). 
For a single Gaussian direction \(z_t\sim\mathcal N(0,I_k)\), the one-sided estimator is
\[
    \widehat g_t
    =
    \frac{\phi(\beta_t+\mu z_t)-\phi(\beta_t)}{\mu}z_t .
\]
For small \(\mu\), the leading term is
$
    a_t=\langle g_t,z_t\rangle z_t .
$
This term is unbiased, since \(\mathbb E[a_t]=g_t\). 
Its second moment is
\[
    \mathbb E\|a_t\|^2
    =
    \mathbb E[\langle g_t,z_t\rangle^2\|z_t\|^2]
    =
    (k+2)\|g_t\|^2 .
\]
Therefore,
\[
    \mathbb E\|a_t-g_t\|^2
    =
    \mathbb E\|a_t\|^2-\|g_t\|^2
    =
    (k+1)\|g_t\|^2 .
\]
With \(q\) independent perturbation directions, define \(\bar a_t=q^{-1}\sum_{j=1}^q a_t^{(j)}\). 
Then
\[
    \mathbb E[\bar a_t]=g_t,
    \qquad
    \mathbb E\|\bar a_t-g_t\|^2
    =
    \frac{k+1}{q}\|g_t\|^2 .
\]
Thus, the leading variance of the ZO gradient estimator scales linearly with the perturbation dimension and decreases as \(1/q\) with multi-query averaging.

We now analyze the convergence of ZO-Act on the objective \(\phi(\beta)=F(\theta_0+U\beta)\). 
The analysis follows the standard analysis framework for gradient descent with a zeroth-order gradient estimator. 

\begin{theorem}[Informal convergence of ZO-Act]
\label{thm:zoact-convergence-informal}
Let \(\phi(\beta)=F(\theta_0+U\beta)\) be the objective optimized by ZO-Act, where \(\beta\in\mathbb{R}^k\) collects all trainable coefficient matrices. 
Assume that $F$ is \(L_F\)-smooth and $\phi$ is lower bounded by \(\phi_k^{\inf}\). 
When ZO-Act applies gradient descent using a \(q\)-query one-sided Gaussian zeroth-order estimator with perturbation scale $\mu$ in the coefficient space, the iterates satisfy
\[
\small
\begin{aligned}
    \frac{1}{T}\sum_{t=0}^{T-1}
    \mathbb{E}\|\nabla\phi(\beta_t)\|^2
    \le\;&
    O\left(
        \frac{
            L_F
            \left(1+\tfrac{k}{q}\right)
            \left(\phi(\beta_0)-\phi_k^{\inf}\right)
        }{T}
    \right) \\
    &+
    O\left(
        L_F^2\mu^2 k^3
    \right).
\end{aligned}
\]

\end{theorem}

The proof is shown in Appendix \ref{sec:app-proof}.
Theorem~\ref{thm:zoact-convergence-informal} highlights the main optimization benefit of ZO-Act. 
Full-weight ZO corresponds to the special case \(k=d\), where \(d=\sum_{\ell=1}^{L}m_\ell n_\ell\) is the full adapted weight dimension. 
ZO-Act instead perturbs only the coefficient matrices, giving \(k=\sum_{\ell=1}^{L}r n_\ell\). 
Since \(r\ll m_\ell\), we have \(k\ll d\), so the variance-dependent term that governs the convergence of \(\phi\) is substantially reduced, yielding a lower-variance and more stable ZO estimator.

The second term is the finite-difference error induced by the one-sided estimator with nonzero perturbation scale \(\mu\). 
Under smoothness-only assumptions, this term scales as \(O(L_F^2\mu^2 k^3)\). 
For full-weight ZO, the same term scales as \(O(L_F^2\mu^2 d^3)\). 
Since ZO-Act uses the much smaller coefficient dimension \(k\ll d\), it can dramatically reduce this finite-difference noise. 
Thus, the low-dimensional activation-informed parameterization reduces both the leading variance term and the higher-order finite-difference error, making the ZO estimator substantially more stable.

This improvement comes with a subspace approximation bias. 
ZO-Act optimizes the restricted objective \(\phi(\beta)=F(\theta_0+U\beta)\), rather than directly optimizing \(F(\theta)\) over the full adapted weight space. 
Therefore, ZO-Act can only reduce the component of the full gradient that lies in the activation-informed update subspace. 
Let \(P_U=UU^\top\) be the projector onto this subspace.
The theorem controls the projected component through \(\nabla\phi(\beta_t)=U^\top\nabla F(\theta_t)\), while the residual term \(\|(I-P_U)\nabla F(\theta_t)\|^2\) measures the bias introduced by restricting updates to the activation-informed subspace. 
Thus, ZO-Act trades a controlled subspace bias for a much lower-variance zeroth-order estimator. In large language model fine-tuning, this subspace bias is often mitigated by the low-rank structure of adaptation. 
Effective task-specific updates are known to be highly structured \citep{zhao2024galore} and are often well captured by low-dimensional parameterizations such as LoRA \citep{Hu2021LoRALA,gurses2025diablo}. 
ZO-Act leverages this phenomenon in a data-informed way by choosing the subspace from dominant input activation directions, where layerwise gradients are expected to concentrate; we verify this gradient concentration empirically in Section~\ref{sec:subspace-alignment}. 
The method can thus preserve important update directions while greatly reducing ZO estimation noise.

To further understand the subspace bias, consider a least-squares layerwise approximation. 
Let \(\mat R=\mat Y-\mat X\mat W_0\) be the residual target. 
The full update problem minimizes \(\|\mat X\Delta\mat W-\mat R\|_F^2\), while ZO-Act restricts the update to \(\Delta\mat W=\mat V\mat B\) and minimizes \(\|\mat X\mat V\mat B-\mat R\|_F^2\). 
If \(P_X\) and \(P_{XV}\) denote the projectors onto \(\operatorname{col}(\mat X)\) and \(\operatorname{col}(\mat X\mat V)\), respectively, then the restricted optimality gap is
\[
    \phi_V^\star-f^\star
    =
    \|(P_X-P_{XV})\mat R\|_F^2 .
\]
If \(\mat X=\sum_i\sigma_i u_i v_i^\top\) and \(\mat V=\mat V_r\) contains the top \(r\) right singular vectors of \(\mat X\), then \(\operatorname{col}(\mat X\mat V_r)=\operatorname{span}\{u_1,\ldots,u_r\}\). 
Therefore,
\[
    \phi_V^\star-f^\star
    =
    \sum_{i>r}
    \|u_i^\top \mat R\|_2^2 .
\]
Thus, the approximation bias is small when the residual target \(\mat R\) has little energy along the discarded left singular directions of the activation matrix. 
When the activation spectrum is concentrated, the dominant left singular directions often capture the most influential output variations induced by input activations, making this residual energy small in practice. 
This supports the use of activation-informed low-rank subspaces: a small rank can preserve the main update directions while substantially reducing the variance of zeroth-order perturbations.


\vspace{-5pt}
\section{Experiments}
\label{sec:exp}


We evaluate ZO-Act on forward-only fine-tuning tasks for large language models. 
Our experiments cover three settings: language understanding and question answering, commonsense reasoning, and quantized LLM fine-tuning. 
For full-precision models, we consider Llama-3-8B \citep{grattafiori2024llama} and OPT-13B \citep{zhang2022opt}. 
For quantized fine-tuning, we evaluate an INT4 quantized Llama-3-8B model, denoted as Llama-3-8B-w4, which is quantized by MagR \citep{zhang2024magr}.

To ensure a fair comparison with existing ZO baselines, we match or slightly reduce the total number of forward passes used by ZO-Act compared with the corresponding baseline methods whenever possible. 
Unless otherwise specified, ZO-Act uses Gaussian perturbations with $q=8$ perturbation directions per update step and perturbation magnitude $\mu=10^{-3}$. We use rank $r=1$ for full-precision models and rank $r=32$ for the INT4 quantized model. The rank-one subspace is sufficient for full-precision fine-tuning and gives the lowest perturbation dimension, while the quantized model benefits from a moderately larger subspace to compensate for the reduced capacity.
We use the forward-difference estimator and optimize the low-rank coefficient matrices with Adam. 
Full details are shown in Appendix \ref{sec:app-exp}.

\subsection{Language Understanding}
\label{sec:exp-language-qa}

\begin{table*}[t]
\centering
\small
\setlength{\tabcolsep}{2.8pt}
\begin{tabular}{lcccccc|cccccc}
\toprule
\multirow{2}{*}{Method}
& \multicolumn{6}{c|}{Llama-3-8B}
& \multicolumn{6}{c}{OPT-13B} \\
\cmidrule(lr){2-7} \cmidrule(lr){8-13}
& SST-2 & RTE & CB & BoolQ & WiC & SQuAD
& SST-2 & RTE & CB & BoolQ & WiC & SQuAD \\
\midrule
\textbullet~Adam 
& 96.0 & 92.0 & 92.0 & 86.6 & 72.6 & 90.4
& 95.3 & 80.9 & 94.6 & 83.5 & 66.3 & 89.5 \\
\textbullet~LoRA 
& 95.0 & 80.9 & 73.2 & 86.4 & 70.7 & 89.4
& 94.8 & 78.3 & 69.6 & 80.2 & 64.3 & 88.0 \\
MeZO 
& 92.7 & 74.4 & 69.6 & 76.7 & 57.8 & 86.7
& 91.4 & 66.1 & 66.0 & 66.1 & 59.4 & 81.8 \\
S-MeZO 
& 92.1 & 69.7 & 69.6 & 80.5 & 56.9 & 87.5
& 90.4 & 63.5 & 69.6 & 66.4 & 58.8 & 80.8 \\
HiZOO 
& 93.5 & 75.1 & 69.6 & 80.0 & 59.7 & 87.3
& 92.1 & 69.3 & 69.6 & 67.6 & 59.4 & 82.1 \\
LOZO 
& 92.5 & 66.8 & 69.6 & 79.4 & 55.8 & 89.0
& 91.7 & 70.4 & 69.6 & 71.9 & 60.2 & 84.9 \\
SubZero 
& 92.1 & 71.4 & 67.9 & 82.0 & 58.8 & 88.3
& 92.1 & 71.8 & 71.4 & 70.8 & 60.8 & 84.5 \\
Subspace-MeZO 
& 92.3 & 68.6 & 69.6 & 80.0 & 62.9 & 84.5
& 91.7 & 70.7 & 71.4 & 68.1 & \textbf{61.7} & 83.5 \\
AGZO & 93.4 & 82.7 & 71.4 & 84.5 & 62.4 & {90.3}& 
89.8 & 67.5 & 66.1 & 68.1 & 56.0 & \textbf{85.4} \\
ZO-Muon 
& 94.3 & 81.2 & 69.6 & 82.9 & 65.2 & 88.2
& 92.5 & 72.9 & 71.4 & 72.4 & \textbf{61.7} & 84.5 \\
\midrule
ZO-Act Full Adam
& \textbf{94.6} & \textbf{87.0} & {89.3} & \textbf{85.1} & \textbf{69.1} & {89.1}
& \textbf{94.0} & \textbf{74.0} & \textbf{75.0} & \textbf{73.1} & 59.4 & {85.1} \\
ZO-Act Full SGD & 94.0 &  83.8 & \textbf{91.1} & 83.8 & 64.7&\textbf{90.7} & 92.4 & 64.6 & 71.4 & 70.0 & 58.2 & 83.9 
\\
ZO-Act INT4 Adam
& 93.5 & 83.8 & 82.1 & 83.8 & 60.7 & 88.2
& -- & -- & -- & -- & -- & -- \\
\bottomrule
\end{tabular}
\caption{
Language understanding and question answering results of various ZO fine-tuning methods. 
}
\label{tab:language_qa}
\end{table*}

\paragraph{Models and Datasets.}
Following the setting of \citet{lang2026powering}, ZO-Act is evaluated on standard language understanding and question answering tasks. 
The evaluation includes six benchmarks: SST-2, RTE, CB, BoolQ, WiC, and SQuAD in SuperGLUE \citep{wang2019superglue}. 
F1 score is reported for SQuAD, and accuracy is reported for the remaining tasks. 
Experiments are conducted on two full-precision LLMs, Llama-3-8B and OPT-13B, as well as an INT4 quantized Llama-3-8B model.

\paragraph{Baselines.}
We compare ZO-Act with representative zeroth-order fine-tuning methods, including MeZO \citep{malladi2023fine}, S-MeZO \citep{liu2026sparse}, HiZOO \citep{zhao2025second}, LOZO \citep{chen2025enhancing}, SubZero \citep{yu2024subzero}, Subspace-MeZO \citep{lang2026powering}, AGZO \citep{lin2026agzo}, and ZO-Muon \citep{lang2026powering}. 
The baseline results, except for AGZO, are taken from the ZO-Muon paper. 
The AGZO results are obtained using our implementation.

\paragraph{Results.}
Table~\ref{tab:language_qa} reports the language understanding and question answering results. 
ZO-Act with Adam achieves strong performance across both model families. 
On Llama-3-8B, ZO-Act obtains the best full-precision ZO result on SST-2, RTE, BoolQ, and WiC, and achieves especially large gains on RTE and CB compared with prior ZO baselines. 
For example, compared with ZO-Muon, ZO-Act improves RTE from 81.2 to 87.0 and CB from 69.6 to 89.3. 
On OPT-13B, ZO-Act also achieves the best full-precision ZO result on SST-2, RTE, CB, and BoolQ, and remains competitive on WiC and SQuAD. 
These results show that the activation-informed space provides effective update directions and leads to robust forward-only fine-tuning performance across different model architectures.

ZO-Act also remains effective when applied to the INT4 quantized Llama-3-8B model. 
Although the INT4 results are generally lower than the full-precision ZO-Act results, the performance remains comparable to or stronger than many full-precision ZO baselines. 
These results suggest that ZO-Act can preserve strong task adaptation ability even when the pretrained model is quantized, supporting its suitability for memory-constrained fine-tuning.

Finally, comparing ZO-Act with Adam and SGD highlights the importance of optimizing the explicit coefficient matrices with a momentum-based optimizer. 
Replacing Adam with SGD leads to noticeably worse performance on most tasks, especially on OPT-13B. 
This confirms one practical advantage of the ZO-Act parameterization: by exposing lightweight trainable coefficient matrices, it allows standard optimizers such as Adam to be directly applied to ZO gradient estimates.

\subsection{Commonsense Reasoning}
\label{sec:exp-commonsense}

\begin{table*}[ht]
\centering
\small
\setlength{\tabcolsep}{4.2pt}
\begin{tabular}{lccccccccc}
\toprule
Method 
& ARC-c & ARC-e & BoolQ & HellaS & OBQA & PIQA & SIQA & WinoG & Average \\
\midrule
Zero-shot &10.0 & 12.5 & 61.9& 12.3 & 23.4 & 47.4 & 2.2 & 0.0 & 21.2 \\
\midrule
LOZO & 51.9 & 69.6 & 62.6 & 48.3 & 53.0 & 70.2 & 54.5 & 50.6 & 57.6\\
AGZO & 65.0 & 83.3 & 64.5& 70.7 &60.8 &\textbf{79.9} &63.5 &57.7 &68.2\\
SubZero 
& 60.2 & 78.3 & 62.5 & 68.0 & 53.6 & 76.7 & 59.2 & 52.6 & 63.9 \\
HiZOO & 60.3 & 79.6 & 63.6 & 65.6 & 58.2 & 73.1 & 58.6 & 53.5 & 64.1\\
ZO-Muon 
& 61.4 & 78.0 & 63.0 & 63.6 & 50.8 & 76.7 & 57.3 & 53.6 & 63.1 \\
\midrule
ZO-Act 
& \textbf{68.7} & \textbf{85.9} & \textbf{64.7} & \textbf{76.4} & \textbf{65.6} 
& {79.6} & \textbf{67.3} & \textbf{58.4} & \textbf{70.8} \\
ZO-Act INT4 & 60.5 & 80.3 & 62.2 & 69.9 & 60.2 & 77.0 & 61.8 & 53.5 & 65.7 \\ 
\bottomrule
\end{tabular}
\caption{
Commonsense reasoning results on Llama-3-8B. 
}
\vspace{-10pt}
\label{tab:commonsense}
\end{table*}

\paragraph{Models and Datasets.}
We further evaluate ZO-Act on full-data commonsense reasoning tasks. 
The benchmark consists of eight tasks: ARC-Challenge, ARC-Easy \citep{clark2018think}, BoolQ \citep{clark2019boolq}, HellaSwag \citep{zellers2019hellaswag}, OpenBookQA \citep{mihaylov2018can}, PIQA \citep{bisk2020piqa}, SIQA \citep{sap2019social}, and WinoGrande \citep{sakaguchi2021winogrande}. 
We fine-tune Llama-3-8B on the combined training data and report accuracy on each evaluation task individually, as well as the average accuracy. All tasks are evaluated under an open-ended generation protocol with answer-match scoring.

\paragraph{Baselines.}
We compare ZO-Act with representative ZO baselines, including LOZO, AGZO, SubZero, HiZOO, and ZO-Muon. 
ZO-Act and ZO-Muon are trained for 20k steps using forward differences with $q=8$, corresponding to 9 forward passes per step, including one unperturbed forward pass and eight perturbed forward passes. 
The remaining baselines are trained for 100k steps with 2 forward passes per step. 
Thus, ZO-Act uses no more total forward evaluations than these baselines.

\paragraph{Results.}
Table~\ref{tab:commonsense} reports the commonsense reasoning results on Llama-3-8B. 
ZO-Act achieves the best average performance among all evaluated ZO methods and obtains the top result on nearly all sub-tasks. 
Compared with prior baselines, ZO-Act shows consistent improvements across the benchmark, indicating that the activation-informed coefficient space provides effective update directions beyond the tasks considered in Table~\ref{tab:language_qa}. 
When ZO-Act is applied to the INT4 quantized model, it remains competitive with strong full-precision ZO baselines, despite some degradation relative to the full-precision ZO-Act model. 
This suggests that the same ZO-Act algorithm can be effectively used on quantized models while retaining strong adaptation performance.

\subsection{Stability and Gradient Alignment of Activation Subspaces}
\label{sec:subspace-alignment}

\begin{figure}[t]
\centering
\includegraphics[width=0.8\columnwidth]{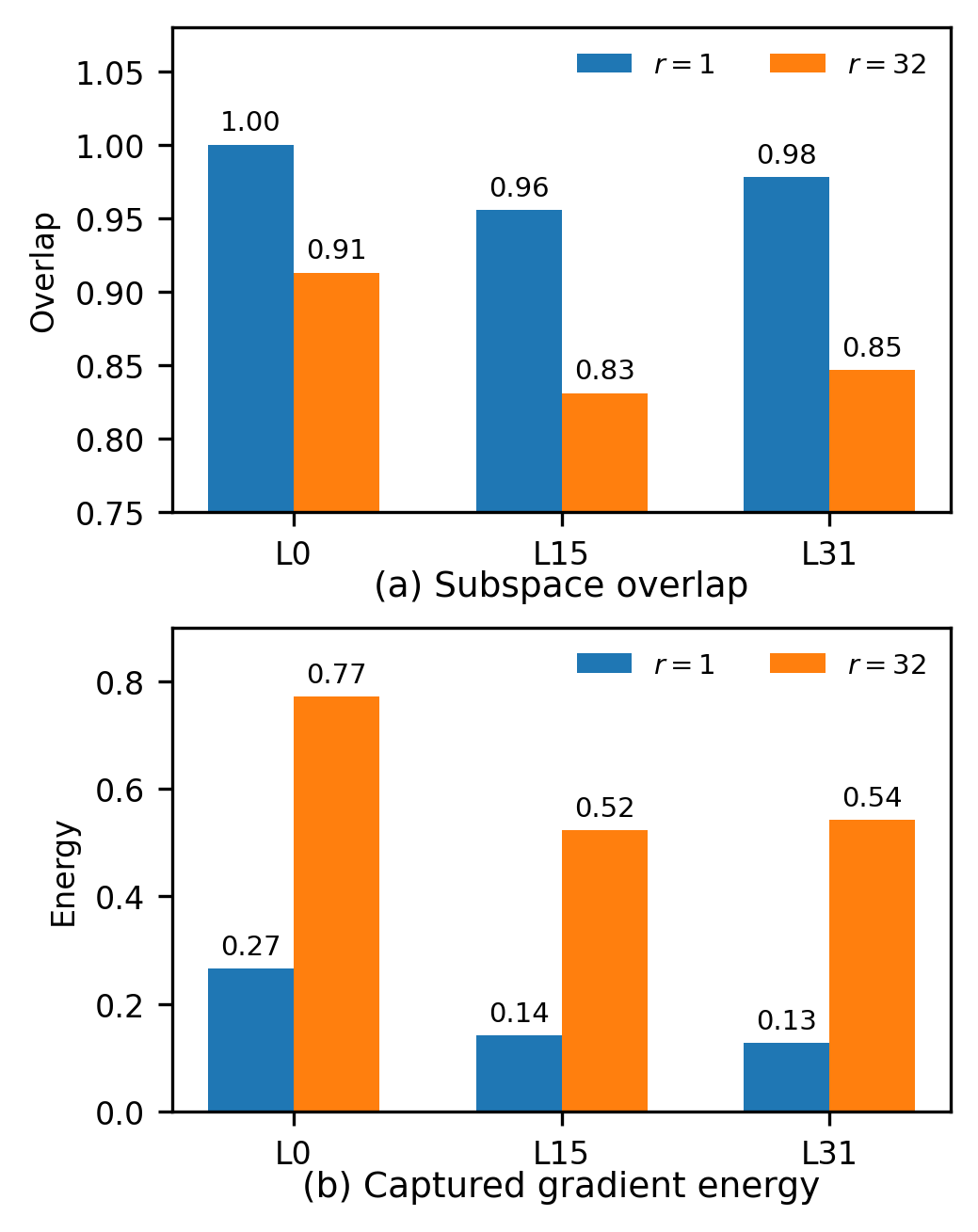}
\caption{
Stability and gradient alignment of activation-informed subspaces at epoch 3 of Llama-3-8B RTE full fine-tuning.
We show the q-projection layers at decoders 0, 15, and 31.
}
\vspace{-15pt}
\label{fig:subspace-diagnostic}
\end{figure}

ZO-Act computes the activation-informed basis once at initialization and keeps it fixed during fine-tuning. 
To validate this design choice, we perform a diagnostic study using a separate full first-order fine-tuning run of Llama-3-8B on RTE for three epochs. We inspect the query projection layers at depths 0, 15, and 31.

For each layer $\ell$, let $\mat V_{\ell,0}\in\mathbb R^{m_\ell\times r}$ denote the top-$r$ right singular vectors of the input activation matrix at initialization, and let $\mat V_{\ell,t}$ denote the corresponding activation basis recomputed at checkpoint $t$. 
We measure the stability of the activation subspace by
$
    \frac{\|\mat V_{\ell,t}^{\top}\mat V_{\ell,0}\|_F^2}{r}.
$
We also measure the fraction of full-gradient energy captured by the fixed initialization subspace:
$
    \frac{\|\mat V_{\ell,0}^T\mat G_{\ell,t}\|_F^2}
         {\|\mat G_{\ell,t}\|_F^2},
$
where $\mat G_{\ell,t}$ is the full weight gradient of layer $\ell$ at checkpoint $t$.

Figure~\ref{fig:subspace-diagnostic} (a) shows the results at the final checkpoint. 
The activation subspace remains stable across all inspected q-projection layers. 
For $r=1$, the overlap is close to one at all three depths, indicating that the dominant activation direction changes little during fine-tuning. 
For $r=32$, the overlap is slightly lower but still remains high.

The fixed activation subspace also captures meaningful gradient energy. The results are shown in Figure \ref{fig:subspace-diagnostic} (b). Even with $r=1$, the initialization subspace captures a clear portion of the full-gradient energy, showing that the dominant activation direction is already gradient-aligned. 
This is important for ZO-Act, since the rank-one setting uses the smallest perturbation dimension and therefore benefits most from variance reduction. 
When the rank is increased to $r=32$, the captured gradient energy further increases, indicating that additional activation directions provide broader gradient coverage. 
Together, these results demonstrate that a very small activation subspace can already identify useful update directions.

\vspace{-7pt}
\subsection{Effect of Subspace Rank}
\label{sec:rank-ablation}

\begin{table}[t]
\centering
\small
\begin{tabular}{ccc}
\toprule
Rank & Llama-3-8B & OPT-13B \\
\midrule
1   & \textbf{87.0} & \textbf{74.0} \\
32  & 83.0 & 68.2 \\
128 & 84.5 & 71.6 \\
256 & 80.9 & 72.6 \\
\bottomrule
\end{tabular}
\caption{
Effect of subspace rank on RTE test accuracy. 
}
\vspace{-15pt}
\label{tab:rank_ablation}
\end{table}

We further study the effect of the activation subspace rank on RTE. 
Table~\ref{tab:rank_ablation} reports the test accuracy of ZO-Act with different ranks on Llama-3-8B and OPT-13B. 
Interestingly, the rank-one subspace achieves the best performance on both models. 
Increasing the rank does not necessarily improve the final test accuracy. These results suggest that the dominant activation direction already captures a highly effective update subspace for ZO fine-tuning. 
The effect of rank is governed by two opposing forces: larger ranks improve subspace coverage and expressiveness, but also increase the coefficient-space perturbation dimension and hence the variance and difficulty of zeroth-order optimization. 
This explains why the accuracy is non-monotonic in the rank rather than uniformly decreasing: very small ranks may underfit the update subspace, intermediate ranks can be variance-dominated, and the adaptive scaling of Adam partially compensates for the added variance at larger ranks. 
The rank-one subspace nonetheless attains the best accuracy on both models while using the smallest perturbation dimension.
In our main experiments, we therefore use a small activation-informed rank, which provides strong performance while preserving the variance-reduction advantage of ZO-Act.

\subsection{Efficiency Comparison.}

\begin{table}[t]
\centering
\small
\setlength{\tabcolsep}{7.0pt}
\begin{tabular}{lcc}
\toprule
Method & Runtime (min) & Memory (GB) \\
\midrule
LOZO        & 44.5 & \textbf{16.1} \\
SubZero     & 45.5 & \textbf{16.1} \\
AGZO        & 47.4 & \textbf{16.1} \\
ZO-Muon     & 42.6 & 19.0 \\
HiZOO       & 94.0 & 30.0 \\
\midrule
ZO-Act      & \textbf{42.0} & \textbf{16.1} \\
\bottomrule
\end{tabular}
\caption{
Runtime and memory on Llama-3-8B. 
}
\vspace{-15pt}
\label{tab:efficiency}
\end{table}

We compare the practical efficiency of ZO-Act with representative ZO fine-tuning baselines in Table~\ref{tab:efficiency}. 
The comparison is conducted on the RTE task using Llama-3-8B, and all methods are evaluated under the same total forward-query budget of 4500 forward passes.

ZO-Act achieves the fastest runtime among the evaluated methods while matching the lowest memory usage. 
Its memory is the same as LOZO, SubZero, and AGZO, and lower than ZO-Muon and HiZOO. 
The one-shot subspace initialization adds negligible overhead, taking only \textbf{1.5} seconds before fine-tuning. 
This indicates that the activation-informed coefficient-space design introduces little practical overhead while retaining the memory efficiency of forward-only fine-tuning. 
Moreover, when ZO-Act is applied to the INT4 quantized model, the memory usage is further reduced to \textbf{5.8GB}, demonstrating its advantage for low-bit fine-tuning.

\section{Conclusion}

We proposed ZO-Act, a one-shot activation-informed zeroth-order fine-tuning method for large language models. 
ZO-Act uses input activations to construct a fixed low-dimensional subspace and optimizes only lightweight coefficient matrices within this subspace. 
By reducing the perturbation dimension, this design lowers the variance of ZO gradient estimation, enables standard momentum-based optimizers, improves convergence, and naturally supports quantized LLM fine-tuning.

Our theoretical analysis shows that ZO-Act improves the stability of zeroth-order optimization by reducing the variance-dependent convergence term and the finite-difference error. 
It also clarifies the main trade-off introduced by the activation-informed subspace: ZO-Act gains more stable coefficient-space optimization while relying on the selected subspace to capture useful update directions. 
Empirically, ZO-Act achieves strong performance on Llama-3-8B, OPT-13B, and INT4 Llama-3-8B across language understanding, question answering, and commonsense reasoning tasks, consistently improving over strong ZO fine-tuning baselines. 
These results suggest that activation-informed coefficient-space optimization is an effective and practical approach for forward-only fine-tuning of both full-precision and quantized LLMs.

\section{Limitations}

Although ZO-Act substantially improves over existing ZO fine-tuning baselines, it still does not fully close the gap to first-order fine-tuning. 
In particular, first-order methods can often achieve stronger final performance because they use exact backpropagation gradients rather than noisy zeroth-order estimates. 
ZO-Act reduces the variance of ZO estimation by restricting perturbations to an activation-informed coefficient space, but the updates are still based only on forward loss evaluations and therefore remain less informative than full gradients.

ZO-Act also remains slower than standard first-order fine-tuning in terms of wall-clock training time. 
Each ZO update requires multiple forward evaluations to estimate a gradient direction, while first-order methods obtain gradients through a single forward-backward pass. 
As a result, ZO-Act is most useful in settings where backpropagation is memory-prohibitive, unavailable, or difficult to support, such as inference-oriented or quantized deployment environments. 
Improving the runtime efficiency and further narrowing the final performance gap between ZO and FO fine-tuning remain important directions for future work.

\section{Ethical Considerations}

This work focuses on improving the optimization efficiency of zeroth-order fine-tuning for large language models and does not introduce new ethical risks beyond those already associated with fine-tuning pretrained LLMs. 
All experiments are conducted on publicly available models (Llama-3-8B, OPT-13B) and standard public benchmarks (SuperGLUE and commonsense reasoning datasets), used in accordance with their respective licenses and intended research use. 
We do not collect any new data or involve human subjects.

ZO-Act is a general-purpose fine-tuning method and inherits the limitations and potential harms of its underlying pretrained models, including possible biases, factual errors, and harmful generations. 
Because ZO-Act adapts models using only forward loss evaluations, it does not mitigate or amplify these issues by design; practitioners should apply standard safety and bias evaluations before deploying any adapted model. 
By lowering the memory cost of fine-tuning and supporting quantized backbones, ZO-Act may broaden access to LLM adaptation on resource-constrained hardware. 
We view this increased accessibility as largely beneficial, but note that, as with any fine-tuning technique, it could in principle be used to adapt models toward harmful ends.

\bibliography{custom}


\appendix

\section{Proofs and Analysis}
\subsection{Formal Statement and Proof of Theorem \ref{thm:zoact-convergence-informal}}
\label{sec:app-proof}
\begin{theorem}[Convergence of ZO-Act]
\label{thm:zoact-convergence-full}
Let \(\phi(\beta)=F(\theta_0+U\beta)\), where \(\beta\in\mathbb R^k\) collects all trainable coefficient matrices. 
Assume that \(F\) is \(L_F\)-smooth, \(U^\top U=I\), and \(\phi\) is lower bounded by \(\phi_k^{\inf}\). 
At iteration \(t\), ZO-Act samples \(q\) independent Gaussian directions \(z_t^{(j)}\sim\mathcal N(0,I_k)\) and uses the one-sided estimator
\[
    \widehat g_t
    =
    \frac{1}{q}
    \sum_{j=1}^{q}
    \frac{\phi(\beta_t+\mu z_t^{(j)})-\phi(\beta_t)}{\mu}
    z_t^{(j)} .
\]
Suppose the update is \(\beta_{t+1}=\beta_t-\eta\widehat g_t\). 
If
\[
    0<\eta\le
    \frac{1}{8L_F\left(1+\frac{k+1}{q}\right)},
\]
then
\[
\begin{aligned}
    \frac{1}{T}\sum_{t=0}^{T-1}
    \mathbb E\|\nabla\phi(\beta_t)\|^2
    \le\;&
    \frac{4\left(\phi(\beta_0)-\phi_k^{\inf}\right)}{\eta T} \\
    &+
    \frac{5}{8}L_F^2\mu^2 k(k+2)(k+4).
\end{aligned}
\]
In particular, choosing \(\eta=\frac{1}{8L_F\left(1+\frac{k+1}{q}\right)}\) gives
\[
\small
\begin{aligned}
    \frac{1}{T}\sum_{t=0}^{T-1}
    \mathbb E\|\nabla\phi(\beta_t)\|^2
    \le\;&
    \frac{
        32L_F\left(1+\frac{k+1}{q}\right)
        \left(\phi(\beta_0)-\phi_k^{\inf}\right)
    }{T} \\
    &+
    \frac{5}{8}L_F^2\mu^2 k(k+2)(k+4).
\end{aligned}
\]
\end{theorem}

\begin{proof}
Since \(F\) is \(L_F\)-smooth and \(U^\top U=I\), the restricted objective \(\phi(\beta)=F(\theta_0+U\beta)\) is also \(L_F\)-smooth. 
Indeed, \(\nabla\phi(\beta)=U^\top\nabla F(\theta_0+U\beta)\), and hence
\[
    \|\nabla\phi(\beta)-\nabla\phi(\beta')\|
    \le
    L_F\|\beta-\beta'\|.
\]
Let \(g_t=\nabla\phi(\beta_t)\). 
For each Gaussian direction \(z_t^{(j)}\), define
\[
    \widehat g_t^{(j)}
    =
    \frac{\phi(\beta_t+\mu z_t^{(j)})-\phi(\beta_t)}{\mu}
    z_t^{(j)},
\]
\[
    \widehat g_t=\frac{1}{q}\sum_{j=1}^{q}\widehat g_t^{(j)} .
\]
By \(L_F\)-smoothness of \(\phi\), for any \(z\),
\[
    \phi(\beta_t+\mu z)
    =
    \phi(\beta_t)
    +
    \mu\langle g_t,z\rangle
    +
    R_t(z),
\]
\[
    |R_t(z)|\le \frac{L_F\mu^2}{2}\|z\|^2 .
\]
We write $
    \widehat g_t^{(j)}
    =
    a_t^{(j)}+b_t^{(j)}$
where,
    $a_t^{(j)}=\langle g_t,z_t^{(j)}\rangle z_t^{(j)},$
    and 
    $b_t^{(j)}=\frac{R_t(z_t^{(j)})}{\mu}z_t^{(j)}.$

The remainder term satisfies \(\|b_t^{(j)}\|\le (L_F\mu/2)\|z_t^{(j)}\|^3\). 
Let \(\bar a_t=q^{-1}\sum_{j=1}^{q}a_t^{(j)}\) and \(\bar b_t=q^{-1}\sum_{j=1}^{q}b_t^{(j)}\), so that \(\widehat g_t=\bar a_t+\bar b_t\).

We first lower bound the expected descent direction. 
Since \(z_t^{(j)}\sim\mathcal N(0,I_k)\), we have \(\mathbb E[a_t^{(j)}]=g_t\), and hence \(\mathbb E[\bar a_t]=g_t\). 
Moreover,
\[
    \mathbb E\|\bar b_t\|
    \le
    \frac{1}{q}\sum_{j=1}^{q}\mathbb E\|b_t^{(j)}\|
    \le
    \frac{L_F\mu}{2}\mathbb E\|z\|^3 .
\]
It holds \(\mathbb E\|z\|^3\le (\mathbb E\|z\|^6)^{1/2}\) and \(\mathbb E\|z\|^6=k(k+2)(k+4)\). We define \(M_k=k(k+2)(k+4)\),
Then \(\mathbb E\|\bar b_t\|\le (L_F\mu/2)\sqrt{M_k}\). 
Thus,
\[
\begin{aligned}
    \mathbb E\langle g_t,\widehat g_t\rangle
    &=
    \mathbb E\langle g_t,\bar a_t\rangle
    +
    \mathbb E\langle g_t,\bar b_t\rangle \\
    &\ge
    \|g_t\|^2
    -
    \|g_t\|\mathbb E\|\bar b_t\| \\
    &\ge
    \|g_t\|^2
    -
    \left(
        \frac{1}{2}\|g_t\|^2
        +
       \frac{1}{2} (\mathbb E\|\bar b_t\|)^2
    \right) \\
    &=
    \frac{1}{2}\|g_t\|^2
    -
    \frac{L_F^2\mu^2}{8}M_k .
\end{aligned}
\]

Next, we upper bound the second moment of \(\widehat g_t\). 
For the leading Gaussian term,
\[
    \mathbb E\|a_t^{(j)}\|^2
    =
    \mathbb E[\langle g_t,z\rangle^2\|z\|^2]
    =
    (k+2)\|g_t\|^2 .
\]
Therefore, since the \(a_t^{(j)}\)'s are independent and each has mean \(g_t\),
\[
\begin{aligned}
    \mathbb E\|\bar a_t\|^2
    &=
    \|\mathbb E\bar a_t\|^2
    +
    \mathbb E\|\bar a_t-\mathbb E\bar a_t\|^2 \\
    &=
    \|g_t\|^2
    +
    \frac{1}{q}\mathbb E\|a_t^{(j)}-g_t\|^2 \\
    &=
    \left(1+\frac{k+1}{q}\right)\|g_t\|^2 .
\end{aligned}
\]
For the finite-difference remainder, Jensen's inequality gives
\[
    \mathbb E\|\bar b_t\|^2
    \le
    \frac{1}{q}\sum_{j=1}^{q}\mathbb E\|b_t^{(j)}\|^2
    \le
    \frac{L_F^2\mu^2}{4}M_k .
\]
Using \(\|\bar a_t+\bar b_t\|^2\le 2\|\bar a_t\|^2+2\|\bar b_t\|^2\), we obtain
\[
    \mathbb E\|\widehat g_t\|^2
    \le
    2\left(1+\frac{k+1}{q}\right)\|g_t\|^2
    +
    \frac{L_F^2\mu^2}{2}M_k .
\]
Let \(A_k=1+(k+1)/q\), then we have
\[
    \mathbb E\|\widehat g_t\|^2
    \le
    2A_k\|g_t\|^2
    +
    \frac{L_F^2\mu^2}{2}M_k .
\]

By smoothness of \(\phi\), the update \(\beta_{t+1}=\beta_t-\eta\widehat g_t\) satisfies
\[
    \phi(\beta_{t+1})
    \le
    \phi(\beta_t)
    -
    \eta\langle g_t,\widehat g_t\rangle
    +
    \frac{L_F\eta^2}{2}\|\widehat g_t\|^2 .
\]
Taking conditional expectation and substituting the two bounds above gives
\[
\begin{aligned}
    \mathbb E_t[\phi(\beta_{t+1})]
    \le\;&
    \phi(\beta_t)
    -
    \eta
    \left(
        \frac{1}{2}\|g_t\|^2
        -
        \frac{L_F^2\mu^2}{8}M_k
    \right) \\
    &+
    \frac{L_F\eta^2}{2}
    \left(
        2A_k\|g_t\|^2
        +
        \frac{L_F^2\mu^2}{2}M_k
    \right).
\end{aligned}
\]
Rearranging,
\[
\begin{aligned}
    \mathbb E_t[\phi(\beta_{t+1})]
    \le\;&
    \phi(\beta_t)
    -
    \left(
        \frac{\eta}{2}
        -
        L_F\eta^2A_k
    \right)\|g_t\|^2 \\
    &+
    \frac{\eta L_F^2\mu^2}{8}M_k
    +
    \frac{L_F^3\eta^2\mu^2}{4}M_k .
\end{aligned}
\]
If \(0<\eta\le \frac{1}{8L_FA_k}\), then \(\eta/2-L_F\eta^2A_k\ge \eta/4\). 
Also, since \(A_k\ge 1\), we have \(L_F\eta\le 1/8\), and hence
\[
    \frac{L_F^3\eta^2\mu^2}{4}M_k
    \le
    \frac{\eta L_F^2\mu^2}{32}M_k .
\]
Therefore,
\[
    \mathbb E_t[\phi(\beta_{t+1})]
    \le
    \phi(\beta_t)
    -
    \frac{\eta}{4}\|g_t\|^2
    +
    \frac{5\eta L_F^2\mu^2}{32}M_k .
\]
Taking total expectation and summing from \(t=0\) to \(T-1\), we get
\[
    \frac{\eta}{4}\sum_{t=0}^{T-1}
    \mathbb E\|g_t\|^2
    \le
    \phi(\beta_0)-\mathbb E[\phi(\beta_T)]
    +
    \frac{5\eta L_F^2\mu^2}{32}M_k T .
\]
Since \(\phi(\beta_T)\ge \phi_k^{\inf}\), dividing by \(\eta T/4\) yields
\[
    \frac{1}{T}\sum_{t=0}^{T-1}
    \mathbb E\|\nabla\phi(\beta_t)\|^2
    \le
    \frac{4(\phi(\beta_0)-\phi_k^{\inf})}{\eta T}
    +
    \frac{5}{8}L_F^2\mu^2M_k .
\]
This proves the theorem.
\end{proof}

\section{Experiment Setups}
\label{sec:app-exp}
In all experiments, ZO-Act is applied to all linear layers except the embedding layers and task-specific linear heads.

\paragraph{Language understanding and question answering.}
Following the setting of \citet{lang2026powering}, ZO-Act is evaluated on six standard language understanding and question answering benchmarks: SST-2, RTE, CB, BoolQ, WiC, and SQuAD in SuperGLUE \citep{wang2019superglue}. 
F1 score is reported for SQuAD, while accuracy is reported for all other tasks. 
Experiments are conducted on two full-precision LLMs, Llama-3-8B and OPT-13B, as well as an INT4 quantized Llama-3-8B model. 
Following prior ZO fine-tuning work \citep{malladi2023fine}, we randomly sample 1{,}000 training examples and 1{,}000 test examples for each task, and use the same prompts as MeZO \citep{malladi2023fine}.

\paragraph{Commonsense reasoning.}
The commonsense reasoning benchmark consists of eight tasks: ARC-Challenge, ARC-Easy \citep{clark2018think}, BoolQ \citep{clark2019boolq}, HellaSwag \citep{zellers2019hellaswag}, OpenBookQA \citep{mihaylov2018can}, PIQA \citep{bisk2020piqa}, SocialIQA \citep{sap2019social}, and WinoGrande \citep{sakaguchi2021winogrande}. 
ZO-Act is fine-tuned on the combined training data using Llama-3-8B, and performance is reported as accuracy on each individual task as well as the average accuracy across all tasks.

The hyperparameters are reported in Table \ref{tab:hyper-llama3}, Table \ref{tab:hyper-llama3-int4}, and Table \ref{tab:hyper-opt13b}.

\begin{table*}[t]
\centering
\small
\setlength{\tabcolsep}{4pt}
\begin{tabular}{lccccccc}
\toprule
 & CB & RTE & WiC & SST-2 & BoolQ & SQuAD & Commonsense \\
\midrule
Rank & \multicolumn{7}{c}{$1$}  \\
Optimizer & \multicolumn{6}{c}{Adam / SGD} & Adam \\
LR (Adam) 
& $5\times10^{-5}$ & $5\times10^{-5}$ & $4\times10^{-5}$ & $3\times10^{-5}$ 
& $5\times10^{-5}$ & $2.5\times10^{-5}$ & $2\times10^{-5}$ \\
LR (SGD) 
& $3\times10^{-6}$ & $3\times10^{-6}$ & $3\times10^{-6}$ & $3\times10^{-6}$ 
& $3\times10^{-6}$ & $5\times10^{-6}$ & -- \\
Scheduler & \multicolumn{7}{c}{constant with warmup} \\
Weight decay & \multicolumn{7}{c}{$0$} \\
$\mu$ & \multicolumn{7}{c}{$10^{-3}$} \\
\bottomrule
\end{tabular}
\caption{Hyperparameters for ZO-Act on Llama-3-8B.}
\label{tab:hyper-llama3}
\end{table*}

\begin{table*}[t]
\centering
\small
\setlength{\tabcolsep}{4pt}
\begin{tabular}{lccccccc}
\toprule
 & CB & RTE & WiC & SST-2 & BoolQ & SQuAD & Commonsense \\
\midrule
Rank & \multicolumn{7}{c}{$32$}  \\
Optimizer & \multicolumn{7}{c}{Adam}  \\
LR (Adam) 
& $6\times10^{-5}$ & $5\times10^{-5}$ & $5\times10^{-5}$ & $3\times10^{-5}$ 
& $3\times10^{-5}$ & $3\times10^{-5}$ & $2\times10^{-5}$ \\
Scheduler & \multicolumn{7}{c}{constant with warmup}  \\
Weight decay & \multicolumn{7}{c}{$0$}  \\
$\mu$ & \multicolumn{7}{c}{$10^{-3}$}  \\
\bottomrule
\end{tabular}
\caption{Hyperparameters for ZO-Act on INT4 quantized Llama-3-8B.}
\label{tab:hyper-llama3-int4}
\end{table*}

\begin{table*}[t]
\centering
\small
\setlength{\tabcolsep}{4pt}
\begin{tabular}{lccccccc}
\toprule
 & CB & RTE & WiC & SST-2 & BoolQ & SQuAD  \\
\midrule
Rank & \multicolumn{6}{c}{$1$}  \\
Optimizer & \multicolumn{6}{c}{Adam / SGD}  \\
LR (Adam) 
& $4\times10^{-5}$ & $2.5\times10^{-5}$ & $2\times10^{-5}$ & $2.5\times10^{-5}$ 
& $2\times10^{-5}$ & $2.5\times10^{-5}$  \\
LR (SGD) 
& $7\times10^{-7}$ & $1\times10^{-6}$ & $1\times10^{-6}$ & $1\times10^{-6}$ 
& $1\times10^{-6}$ & $1\times10^{-6}$  \\
Scheduler & \multicolumn{6}{c}{constant with warmup}  \\
Weight decay & \multicolumn{6}{c}{$0$}  \\
$\mu$ & \multicolumn{6}{c}{$10^{-3}$}  \\
\bottomrule
\end{tabular}
\caption{Hyperparameters for ZO-Act on OPT-13B.}
\label{tab:hyper-opt13b}
\end{table*}

\end{document}